\documentclass[twocolumn, switch]{article} 

\usepackage{preprint}

\usepackage{amsmath, amsthm, amssymb, amsfonts}
\usepackage{bm}
\usepackage{amsmath}
\usepackage{graphicx}
\usepackage{subfigure}
\usepackage{algorithm}
\usepackage{algorithmic}
\usepackage{setspace}
\usepackage{amssymb}
\usepackage{booktabs}
\usepackage{array}
\usepackage{color}
\usepackage{multirow}
\usepackage{multicol}
\usepackage{threeparttable}
\usepackage{url}
\usepackage[page]{appendix}

\usepackage[numbers,square]{natbib}
\usepackage[utf8]{inputenc}	
\usepackage[T1]{fontenc}	
\usepackage{xcolor}		
\usepackage[colorlinks = true,
            linkcolor = blue,
            urlcolor  = blue,
            citecolor = blue,
            anchorcolor = black]{hyperref}	
\usepackage{booktabs} 		
\usepackage{nicefrac}		
\usepackage{microtype}		
\usepackage{lineno}		
\usepackage{float}			

\usepackage{newfloat}
\DeclareFloatingEnvironment[name={Supplementary Figure}]{suppfigure}
\usepackage{sidecap}
\sidecaptionvpos{figure}{c}

\usepackage{titlesec}
\titlespacing\section{0pt}{12pt plus 3pt minus 3pt}{1pt plus 1pt minus 1pt}
\titlespacing\subsection{0pt}{10pt plus 3pt minus 3pt}{1pt plus 1pt minus 1pt}
\titlespacing\subsubsection{0pt}{8pt plus 3pt minus 3pt}{1pt plus 1pt minus 1pt}

\usepackage{graphics}
\usepackage{hyperref}
\usepackage{color}
\usepackage{multirow}
\usepackage{hhline}
\usepackage{subfigure}

\title{Predicting the Popularity of Micro-videos with Multimodal Variational Encoder-Decoder
Framework}

\usepackage{authblk}

\author{Yaochen Zhu}
\author{Jiayi Xie}
\author{Zhenzhong Chen\thanks{\tt{zzchen@ieee.org}}}
\affil{School of Remote Sensing and Information Engineering, Wuhan University}

\begin{document}

\twocolumn[ 
  \begin{@twocolumnfalse} 
  
\maketitle

\begin{abstract}

As an emerging type of user-generated content, micro-video drastically enriches people's entertainment experiences and social interactions. However, the popularity pattern of an individual micro-video still remains elusive among the researchers. One of the major challenges is that the potential popularity of a micro-video tends to fluctuate under the impact of various external factors, which makes it full of uncertainties. In addition, since micro-videos are mainly uploaded by individuals that lack professional techniques, multiple types of noise could exist that obscure useful information. In this paper, we propose a multimodal variational encoder-decoder (MMVED) framework for micro-video popularity prediction tasks. MMVED learns a stochastic Gaussian embedding of a micro-video that is informative to its popularity level while preserves the inherent uncertainties simultaneously. Moreover, through the optimization of a deep variational information bottleneck lower-bound (IBLBO), the learned hidden representation is shown to be maximally expressive about the popularity target while maximally compressive to the noise in micro-video features. Furthermore, the Bayesian product-of-experts principle is applied to the multimodal encoder, where the decision for information keeping or discarding is made comprehensively with all available modalities. Extensive experiments conducted on a public dataset and a dataset we collect from Xigua demonstrate the effectiveness of the proposed MMVED framework. 

\end{abstract}

\vspace{0.4cm}

  \end{@twocolumnfalse} 
] 

\section{Introduction}\label{sec:introduction}
{\let\thefootnote\relax\footnote{This work was supported in part by National Key R\&D Program of China under contract No. 2017YFB1002202. (Corresponding author: Zhenzhong Chen, E-mail:zzchen@ieee.org)}}In recent years, with the prevalence of social media, people grow increasingly enthusiastic about spreading various forms of their creations on the Internet, and the consumption of user-generated contents (UGCs) has gradually become an indispensable daily entertainment activity for most of the populations. Interestingly, for tons of UGCs posted on social media every day, most of them soon fall into oblivion, while some are able to attract lots of attention and disseminate widely, reflected by receiving comparatively large amounts of views, likes, comments, reposts. 

Various internal and external factors may affect the popularity of UGCs, including not only the quality of their content, the influence of their publishers, but even the release timing as well \cite{DBLP:conf/aaai/LiaoXLHLL19, chen2016micro}. Accurate predictions of the popularity level of UGC could be of great benefit: It allows the service providers to make strategic decisions as to the management of their resources, such as caching and network optimization \cite{DBLP:journals/wc/ZengLCYHWL18}, and provides users with more satisfactory personal recommendations \cite{DBLP:journals/access/BielskiT18}. Therefore, the popularity prediction of UGCs has received lots of attention among researchers.

Extensive work has been done in the domain of popularity prediction for traditional forms of UGCs, such as articles \cite{DBLP:conf/aaai/LiaoXLHLL19}, news \cite{DBLP:conf/epia/FernandesV015}, and images \cite{mcparlane2014nobody}. Generally speaking, early studies on popularity prediction of UGC can be divided into two categories, i.e., feature-driven methods and generative (time series-driven) methods \cite{DBLP:journals/ijon/ChenKXM19}. Feature-driven approaches first extracted a large number of features related to the  UGC contents, the profile of users, or the social networks, and trained machine learning models such as Support Vector Machine (SVM) \cite{mcparlane2014nobody} or Random Forest (RF) \cite{DBLP:conf/epia/FernandesV015} to optimize a mapping function from the feature space to a pre-defined popularity space. They mainly focused on feature engineering techniques and could achieve good performance, provided that the extracted features are effective. On the other hand, for analysis of the popularity evolution pattern of UGCs over time, generative approaches utilized the temporal regularities of the popularity curve at the early stage to fit an autoregressive model with the linear or non-linear dynamics to predict its future trends \cite{DBLP:journals/cacm/SzaboH10,  DBLP:journals/ijon/HuHFSN16, DBLP:conf/wsdm/PintoAG13}; most of them relied on strong hypotheses about the popularity accumulation mechanism. Recently, inspired by the outstanding performance of deep learning models, several deep learning-based methods for UGC popularity prediction have been proposed \cite{DBLP:journals/ijon/ChenKXM19, DBLP:journals/access/BielskiT18} in both categories, taking advantage of the massive representative powers of convolutional neural networks (CNNs) to extract more powerful features and recurrent neural network (RNNs) to capture the more complicated temporal relationships. 

However, as an emerging form of UGC, litter efforts have been dedicated to the understanding of micro-videos. Compared to popularity prediction for traditional UGCs or professionally made videos such as movies, predicting the popularity of micro-videos faces its own challenge due to the following factors: (1) \textbf{Heterogeneity}: Aside from a short video that lasts from 6 seconds to 60 seconds, several other components are usually attached, such as background music, titles, and hashtags, each describing the micro-video from a different perspective. Therefore, micro-videos are way richer in information than other forms of UGCs, which renders it more difficult to extract relevant features and fuse them effectively across modalities to explain the observed popularity trend. (2) \textbf{External Uncertainty}: Even if features that are highly informative to the popularity of micro-videos could be properly extracted and fused, since micro-videos are made to gratify the mercurial taste of the massive online audiences, the popularity of a micro-video could vary under the impacts of lots of external uncertainties such as the time of publishing, change of trends or influence from online-celebrities and other social media. Thus, the observed popularity usually appears a certain amount of randomness. (3) \textbf{Internal Noise}: Moreover, compared to the professional teams by whom the movies are made, the uploaders of micro-videos are usually unaware of various expertise such as film grammars to convey affective information to the audience\cite{wang2019video}, or lack the professional devices to shoot the video, which may lead to poor quality in both visual and acoustic contents. Besides, when users upload micro-videos online, they may make up sensational titles and tags that are irrelevant to the content just for the eye-catching effect, and the profile of users could be falsified as well, which makes the textual and user features untrustworthy. 

Therefore, in order to address these challenges, we propose a multimodal variational encoder-decoder (MMVED) framework for micro-video popularity prediction. Unlike the majority of previous approaches where the map from features of a piece of UGC to its popularity level is assumed to be deterministic, MMVED learns a stochastic Gaussian embedding of a micro-video that is informative to the prediction of its future popularity level, while preserves the inherent randomness of the popularity caused by various external uncertainties simultaneously. In particular, MMVED takes the maximization of a variational approximation to the information bottleneck as its objective, such that only relevant cues in the information-rich micro-video features can be extracted into the hidden representation, ignoring the irrelevant and noisy parts. Specifically, in the multimodal encoder, we adopt the Bayesian product-of-experts strategy to fuse the modality-specific embeddings, where both the information heterogeneity and difference of uncertainty of all modalities are comprehensively considered. MMVED could also be shown to be parameter-economic to deal with the modalities missing problem in the test phase.

Please note that a preliminary conference version of this paper has been presented at WWW 2020 \cite{MMVED-WWW2020}. Compared to the initial paper that constructed the MMVED objective by adding the Variational Auto-Encoder (VAE) loss as an ad-hoc regularizer to the vanilla encoder-decoder target, this manuscript rigorously proves that such MMVED objective in essence lower-bounds the information bottleneck objective through the deep variational information bottleneck (D-VIB) theory, which further reveals the underlying information-sifting and denoising mechanism of the proposed framework; Besides, we demonstrate both theoretically and empirically that various form of decoders, such as MLP with cross-entropy loss, MLP with MSE loss or RNN could be incorporated into the MMVED framework for corresponding tasks of UGC popularity prediction, such as popularity classification, regression, and temporal regression, etc. Moreover, more advanced feature engineering techniques are utilized to characterize different aspects of the micro-video, which take advantage of recent advances of deep neural networks in both computer vision and audio processing community.

The remainder of this paper is organized as follows: Section \ref{SEC:BACKGROUNDS} surveys related work concerning the popularity prediction task of UGCs; Section \ref{SEC:METHODOLOGY} expounds the proposed MMVED framework in detail; Section \ref{SEC:FEATURE} describes the adopted feature engineering techniques; Section \ref{SEC:NUS} and \ref{SEC:XIGUA} evaluate the proposed model on two micro-video popularity prediction datasets and analyze experimental results; Finally, section \ref{SEC:CONCLUSION} concludes the paper.

\section{Related Work}
\label{SEC:BACKGROUNDS}

Due to its importance in recommendation, advertising, and many other applications, popularity prediction of UGCs on social media receives considerable attention in both industry and academia. The premiere step to predict the popularity of UGCs is to measure it numerically. Generally, the popularity of a piece of UGC is defined by the volume of positive response it receives, which can be estimated by the number of viewers, likes, comments, and reposts. In practice, the weighted average of any combinations of these indexes could adequately serve as an indicator of popularity levels of UGCs, and they are widely used in researches to construct the groundtruth for UGC popularity prediction tasks.

For popularity prediction of online textual contents, such as hashtags, microblogs, and articles, most studies focused on combining the representative features from the textual contents \cite{DBLP:conf/epia/FernandesV015}, \cite{DBLP:conf/hicss/Doong16} and the social context, such as the way users are linked, to train a popularity predictor. One exemplar work is \cite{ma2013predicting}, where several content features from hashtags of the tweets were fused with contextual features from the user social graph to train multiple outside classifiers. Besides, for popularity prediction of Reddit comments, Zayats and Ostendorf \cite{zayats2018conversation} extracted the textual features from the reviews and combined the user interaction graph with a long short term memory (LSTM) network to model the influence of user connections and temporal evolution to the popularity of comments over time.

As for predicting the popularity of images and long-videos \cite{roy2013towards} (such as those posted on YouTube) where rich information is contained in visual or aural modalities, multimodal learning that fuses information from different views is one of the most commonly used techniques. Li et al. \cite{DBLP:conf/cikm/LiMWLX13} introduced a novel propagation-based popularity prediction method by considering both intrinsic video attractiveness and the underlying propagation structure. Khosla et al. \cite{DBLP:conf/www/KhoslaSH14} explored the relative significance of individual features involving multiple visual features, as well as various social context features, including the number of views and contacts. Cao et al. \cite{DBLP:journals/corr/abs-1906-09032} utilized two coupled neural networks to iteratively capture the cascading effect in information diffusion to predict future popularity. Trzcinski et al. \cite{trzcinski2017predicting} predicted the popularity of the video with both the visual clues and the early popularity pattern of the video after its publish over a certain period. However, none of these methods considered the cold start scenario. In order to deal with such a challenge, McParlane et al. \cite{mcparlane2014nobody} took a content-based strategy and utilized only visual appearance and user context for Flickr image popularity prediction. Bielski et al. \cite{DBLP:journals/access/BielskiT18} predicted the popularity of videos before they are published, by exploiting the spatio-temporal characteristics of videos through a soft self-attention mechanism, and they intuitively interpreted the impact of contents on video popularity by Grad-CAM Algorithm. 

Although an emerging form of UGC, several pioneer work have been dedicated towards the popularity prediction of micro-videos.  Chen et al. \cite{chen2016micro} first employed four types of heterogeneous features, i.e., the visual, acoustic, textual, and social features to describe the characteristics of micro-videos, and proposed a transductive multimodal learning model to regress the popularity index of micro-videos after their release. Trzci{\'n}ski
and Przemys{\l}aw \cite{trzcinski2017predicting}, on the other hand, focused on using the early popularity pattern of the popularity curve to predict the trend for the time after, which extended the micro-video popularity prediction task to an online manner. Afterwards, Jing et al. \cite{DBLP:journals/tkde/JingSNBLW18} thoroughly discussed the detrimental effects of internal noise to the micro-video analysis studies, and they augmented the multi-view learning method with a low-rank constraint, such that only a few principal patterns in micro-video features are allowed to be kept in its representation.

The most salient character that distinguishes our methods from the above ones is that we also consider the inherent uncertainty in popularity and treats all external uncertain factors as randomness, which can be properly preserved when we learn the Gaussian embedding of the micro-videos. Besides, compared with \cite{DBLP:journals/tkde/JingSNBLW18}, which also considered the adverse effect of noise, we resort to combining the representative learning power of the deep neural network and the denoising ability of information bottleneck structure, where for each sample, the information that is allowed to be extracted into the hidden representation is actively learned and dynamically decided. Furthermore, in \cite{DBLP:journals/tkde/JingSNBLW18}, the popularity was a single numerical value, whereas our model is also suitable for tasks where the popularity is represented by a time-related sequence, which is more difficult since it requires understanding  the hidden evolutionary pattern of the popularity trend.

\section{The Multimodal Variational Encoder-Decoder Framework}
\label{SEC:METHODOLOGY}

\subsection{Problem Formulation}

\begin{figure*}
\centering
\includegraphics[scale=0.76]{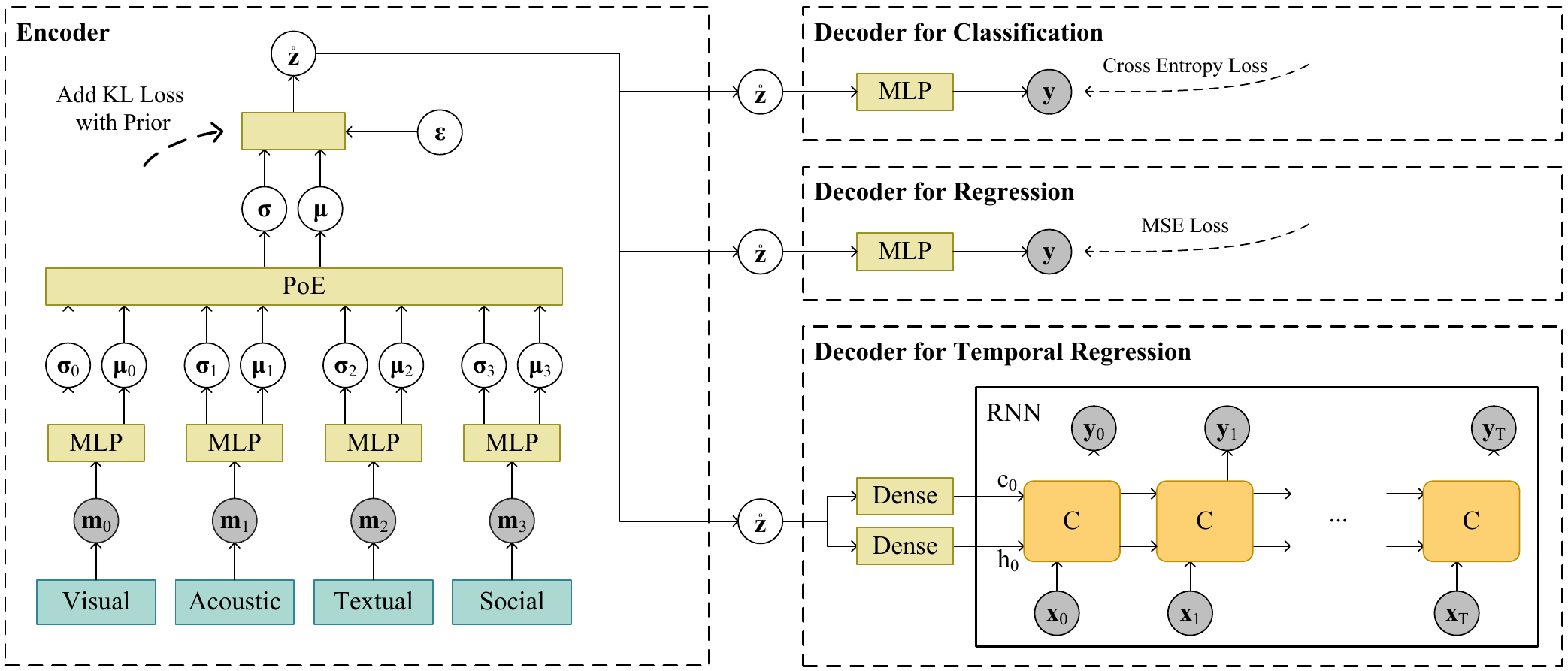}
\caption{The framework of the proposed MMVED model for various UGC popularity prediction tasks, where the observed variables are shaded gray while the hidden ones are white. The forward steps of MMVED consist of three steps. First, features from each modality $m_{i}$ of a micro-video are encoded into modality-specific hidden representations and are fused together according to the Bayesian product-of-experts principle to compute the micro-video level hidden representation $\mathbf{z}$, which follows $\mathcal{N}(\boldsymbol{\mu}, \boldsymbol{\sigma})$. Its KL divergence with the uninformative prior is computed as the "information bottleneck" constraint, and a $\lambda$ is specified to control its "narrowness". Second, $\mathring{\mathbf{z}}$ are sampled from $\mathbf{z}$ with the reparameterization trick. Third, the popularity index of the micro-video $\mathbf{y}$ is decoded from the $\mathring{\mathbf{z}}$ through the suitable decoder.}
\label{FIG:MODEL}
\end{figure*}

Consider we collect a micro-video popularity prediction dataset $\mathcal{ D }$ of the form $\left\{ \left( \mathbf{m} ^ { 1 },  \mathbf{y} ^ { 1 } \right) , \ldots , ( \mathbf{m} ^ { N }, \mathbf{y}  ^ { N }) \right\}$, where $\mathbf{m} ^ {i} = \{m^{i}_{1}, \ldots, m^{i}_{K}\}$ is the features extracted from $K$ modalities, and $\mathbf{y} ^ {i}$ is the corresponding popularity groundtruth which could assume one of the following three forms: $\in \{0, 1\}$ for classification, $\in \mathbb{R}$ for regression, or $ \in \mathbb{R}^{T} $ for temporal regression, our goal is to learn a hidden representation $\mathbf{z}$ of a micro-video based on the integrated information from all the modalities that is highly informative for its potential  popularity level, while being robust to the external uncertain factors and the internal feature noises. The main notations used in section \ref{SEC:METHODOLOGY} are summarized in Table \ref{tab:symbols} for the convenience of reference.

\begin{table}
	\centering
	\caption{Description of Symbols Used in Section \ref{SEC:METHODOLOGY}}
	\label{tab:symbols}
	\begin{tabular}{l|l}  
		\toprule
		Symbol & Description \\
		\midrule
		$K$                    & number of modalities in a micro-video\\
		$N$                    & number of samples of the dataset\\
		$L$                    & number of samples to estimate the gradient\\
		$D$                    & dimension of the Gaussian embedding\\ 
		
		$\mathbf{m}$           & variable for the set of multimodal features\\	
		$\mathbf{y}$           & variable for the popularity groundtruth \\
		$\mathbf{z}$           & variable for the Gaussian embedding \\

		$\boldsymbol{\mu}$         & mean of the Gaussian embedding\\
		$\boldsymbol{\sigma}$      & standard deviation of the Gaussian embedding\\

		$\boldsymbol{\phi}$                 & trainable parameters of the encoder  $q(\mathbf{z}|\mathbf{m})$\\
		$\boldsymbol{\theta}$               & trainable parameters of the decoder  $q(\mathbf{y}|\mathbf{z})$\\

		$\lambda$              & weight in the training objective\\
		
		$\operatorname{MI}$       & mutual information operator\\
		$\operatorname{LLD}$      & expected log-likelihood term of $\operatorname{IBLBO}$\\
		$\operatorname{IBLBO}$    & information bottleneck lower bound\\
		\bottomrule
	\end{tabular}
\end{table}

\subsection{Training Objective}

As the hidden representation $\mathbf{z}$ is derived from micro-video features $\mathbf{m}$ whereas the popularity $\mathbf{y}$ is predicted based on $\mathbf{z}$, their relationship can be represented by a Bayesian chain $\mathbf{m} \rightarrow \mathbf{z} \rightarrow \mathbf{y}$ and can be modeled by an encoder-decoder framework.  Since the mapping from the content of the video to its popularity may not be deterministic due to some unpredictable external factors, it is reasonable to assign $\mathbf{z}$ some randomness, and represent both the encoder and decoder as probabilistic distributions. In this paper we assume the hidden representation $\mathbf{z}$ lies in $D$-dimensional Gaussian space, and use $p(\mathbf{z}|\mathbf{m})$ and $p(\mathbf{y}|\mathbf{z})$ to denote the encoder and  decoder distribution respectively.  

Due to the general low quality of micro-videos, not all information from input modalities are beneficial for prediction of their popularity, and thus the noise of $\mathbf{m}$ should not be extracted into the hidden representation. Towards this end, we first define a common criteria to measure the $relevance$ level of one random variable to another, i.e. the mutual information (MI), which is formulated as follows:

\begin{equation}
\label{EQ:MI}
\mathbf{MI}(\mathbf{a}, \mathbf{b}) = \mathbb { E } _ {p(\mathbf{a}, \mathbf{b})}\left[ \log p(\mathbf{a}, \mathbf{b}) - \log p(\mathbf{a})p(\mathbf{b})\right]
\end{equation}

\noindent Since the learned embedding $\mathbf{z}$ is expected to pick up only cues in the input $\mathbf{m}$ that are \textit{relevant} to the popularity $\mathbf{y}$ while ignoring the noisy distractors, utilizing the MI, we can construct the following constrained optimization objective:

\begin{equation}
\label{EQ:COBJ}
\max \mathbf{MI}(\mathbf{z}, \mathbf{y}), \text {s.t.} \mathbf{MI}(\mathbf{m}, \mathbf{z}) \leq I_{c}
\end{equation}

\noindent where $I_{c}$ is maximum limit of information contained in $\mathbf{z}$ of $\mathbf{m}$. The equation above forces the stochastic embedding $\mathbf{z}$ to be expressive of $\mathbf{y}$ while being compressive of $\mathbf{m}$ at the same time, which provides a mechanism for the model to predict the true popularity level while accessing the minimal amount of information from the input to block the noise. By introducing a Lagrange multiplier $\lambda$, the constrained optimization problem above can be proved to be equivalent to the following unconstrained one, which has the same form as the information bottleneck first proposed in \cite{tishby99information}.

\begin{equation}
\label{EQ:UCOBJ}
\max R_{IB}=\mathbf{MI}(\mathbf{z}, \mathbf{y})-\lambda \cdot \mathbf{MI}(\mathbf{z}, \mathbf{m})
\end{equation}

\noindent Intuitively, $\lambda$ determines the degree of penalty associated with keeping more information of $\mathbf{m}$ in $\mathbf{z}$, which controls the trade-off between minimum usage of information in the input and predicting the popularity target. 

Since the decoder distribution $p(\mathbf{y}|\mathbf{z})$ can take the form of any valid conditional distributions and most of which are not even differentiable, it is intractable to directly calculate the two MI terms in $R _ {IB}$ according to Eq. \ref{EQ:MI} and optimize. Therefore, we resort to the variational approach \cite{blei2017variational}, which posits that the decoder comes from a tractable family of distribution $\mathcal{Q}$ and finds a distribution $q(\mathbf{y}|\mathbf{z})$ in that family that is closest to the optimal decoder distribution  measured by the KL-divergence. According to the deep variational information bottleneck (D-VIB) theory \cite{blei2017variational}, for  $q(\mathbf{y}|\mathbf{z})$, the $R _ {IB}$ can be proved to be lower bounded by: 

\begin{equation}
\label{EQ:IBLBO}
\begin{aligned}
\mathbf{MI}(\mathbf{z}, \mathbf{y})-\lambda \cdot \mathbf{MI}(\mathbf{z}, \mathbf{m}) \geq \operatorname{IBLBO} & \triangleq \\
\mathbb { E } _ {p (\mathbf{m}, \mathbf{y}) } \Bigl[ \mathbb { E } _ {p(\mathbf{z}|\mathbf{m})} \left[\log q (\mathbf{y}|\mathbf{z})\right] - \lambda \cdot \mathbf { KL }  ( p (\mathbf{z}&|\mathbf{m}) \| p (\mathbf{z}) ) \Bigl]
\end{aligned}
\end{equation}

\noindent which we note as the Information Bottleneck Lower BOund (IBLBO). The IBLBO has the same form as the OBJ we used in \cite{MMVED-WWW2020}. However, in \cite{MMVED-WWW2020}, the OBJ is constructed by ad-hocly adding a VAE loss to the encoder-decoder objective. The full proof of IBLBO can be referred to in Appendix \ref{APP:DIB}. Hence, we can optimize IBLBO as a surrogate for the intractable $R _ {IB}$. Note that the first expectation of Eq. \ref{EQ:IBLBO} is taken w.r.t the data distribution $p(\mathbf{m}, \mathbf{y})$. Given that the true data distribution is usually  inaccessible, the empirical data distribution is used instead, and in the rest we would omit the first expectation for simplification. The IBLBO resembles the likelihood-prior trade-off commonly found in Statistics, since the first term in the objective is the expected log-likelihood, which encourages probability density of the encoder to be put where the hidden representation could best explain the observed popularity trend, whereas the second term is the KL-divergence of the encoder distribution with the uninformative prior, which penalizes the deviation of $\mathbf{z}$ from the prior for keeping excessive information in the input $\mathbf{m}$ and serves as a regularizer. We name the model represented by Eq. \ref{EQ:IBLBO} as Multimodal Variational Encoder-Decoder (MMVED) for the rest part of the paper.

\subsection{The Multimodal Encoder}
\label{SEC:MMENCODER}

The stochastic encoder for the micro-video $p(\mathbf{z}|\mathbf{m})$ is parameterized as a deep neural network (DNN), more specifically, a multi-layer perceptron (MLP), which is basically a stack of fully connected layers with intermediate activations. A naive way to construct the encoder network $p(\mathbf{z}|\mathbf{m})$ is to adopt the early fusion strategy \cite{DBLP:conf/mm/SnoekWS05}, which takes the concatenated features from all modalities as its input and outputs the parameters of the encoder distribution, i.e., its mean $\boldsymbol{\mu}$ (semantic part) and and logarithm of standard deviation $\boldsymbol{\sigma}$ (uncertainty and noise part):

\begin{equation}
\label{EQ:NAIVE}
[\boldsymbol{\mu}, \log \boldsymbol{\sigma}] = \operatorname{MLP}\left(\operatorname{concat}\left([{\mathbf{m} _ {1}, \ldots, \mathbf{m} _ {K}}]\right)\right)
\end{equation}

However, such approach was shown by previous work to have unsatisfactory performance for micro-video popularity prediction task \cite{chen2016micro}, since it is unable to account for the relatedness among multiple modalities. Thus, in order to make inference of $\mathbf{z}$ based on the complementary information from all modalities, inspired by the recent advance in multimodal variational inference frameworks \cite{DBLP:conf/nips/WuG18} \cite{zhu2019multimodal}, we assume conditional independence among features from each modality given the hidden embedding, i.e. $(\mathbf{m} _ {i} \perp \mathbf{m} _ {j} \, | \, \mathbf{z})  \quad \forall i,j < K$ and $i \neq j$. Then, as was shown by \cite{DBLP:conf/nips/WuG18}, the joint inference distribution can be factorized into the product of modality-specific encoder distributions $p ( \mathbf{z} | \mathbf{m} ) \propto p(\mathbf{z})\prod _ { i = 1 } ^ { K } p ( \mathbf{z}|\mathbf{m} _ { i } )$, which is a typical Bayesian product-of-experts (PoE) system \cite{DBLP:journals/neco/Hinton02}. We refer interested readers to \cite{DBLP:conf/nips/WuG18} for the full prove of above deductions, and some of the most important steps are included in Appendix \ref{APP:POE} for the self-containment of our paper. Such factorization means that we can first use the modality-specific encoder networks to compute the parameters of the latent representation for each modality as follows:

\begin{equation}
\label{EQ:MOD_SPE_GAUSSIAN}
[\boldsymbol{\mu} _ {i}, \log \boldsymbol{\sigma} _ {i}] = \operatorname{MLP _ {i}}(\mathbf{m} _ {i}), \quad i \in {1, \ldots, K} 
\end{equation}

\noindent Then, by using the property of multivariate Gaussian, the parameters of the hidden representation for the whole micro-video can be calculated as:

\begin{equation}
\begin{aligned}
\label{EQ:WHOLE_GAUSSIAN}
\boldsymbol{\mu}  &= \sum \left(\boldsymbol{\mu} _ {i} \odot \left(1 / \boldsymbol{\sigma} _ {i} ^ {2}\right) / \sum{1 / \boldsymbol{\sigma} _ {i} ^ {2}}\right) \\
\boldsymbol{\sigma} &= \operatorname{sqrt}\left(1 / \sum{1 / \boldsymbol{\sigma} _ {i} ^ {2}}\right) \\ 
\end{aligned}    
\end{equation}

\noindent where $\odot$ is the element-wise product operation.  \textbf{Eq. \ref{EQ:WHOLE_GAUSSIAN} shows that, under the conditional independence assumption,  the semantic part $\boldsymbol{\mu}$ of the video-level hidden representation $\mathbf{z}$ is in essence the average of modality-specific $\boldsymbol{\mu}_{i}$ weighted by the reciprocal of the corresponding variance $\boldsymbol{\sigma} _ {i} ^ {2}$. With such calculation, for experts with greater precision, which indicates that the information from their associated modalities bear less uncertainty and more relevance to the popularity prediction end, they will have more influence (larger weights) over the overall hidden representation than those with higher variance.} Another by-product of adopting such factorization is that it is easy to deal with modality missing problem in the test phase, since under such cases the only change to our framework is to remove the corresponding modality-specific encoder networks and re-compute the PoE distribution with available modalities, eliminating the needs of additional inference networks and multi-stage training regimes. 

\subsection{The Variational Decoders}

The choice for the tractable family $\mathbf{Q}$ of the decoder distributions $q(\mathbf{y} | \mathbf{z})$ is conditioned on the specific popularity prediction task at hand, and we will discuss three commonly faced tasks and the corresponding decoder structures.

\subsubsection{Classification: MLP with Cross-entropy Loss}

If the popularity groundtruth $\mathbf{y}$ is represented as a binary variable with 1 indicating popular and 0 otherwise, the family $\mathbf{Q}$ of the decoder distributions $q(\mathbf{y} | \mathbf{z})$ can be chosen as an MLP where the output $\mathbf{o} _ {t}$ is squashed through a Sigmoid function as follows:

\begin{equation}
 q(\mathbf{y} | \mathbf{z}) = \operatorname{Sigmoid}(\operatorname{MLP}(\mathbf{z})) 
\end{equation}

\noindent where $\operatorname{Sigmoid}(\mathbf{x}) = 1 / \left( 1 + e ^ { - \mathbf{x} } \right)$. The maximization of the expected $\log q(\mathbf{y} | \mathbf{z})$ term in Eq. \ref{EQ:IBLBO} is equivalent to minimization of the binary cross-entropy loss between the MLP prediction and the groundtruth.

\subsubsection{Regression: MLP with MSE Loss}

However, representing the popularity level of a micro-video as a binary variable is too coarse in granularity, which fails to discriminate micro-videos with different popularity degrees. Therefore, it's more common to define the $\mathbf{y}$ as a continuous variable. Under such circumstances, in order to calculate $q(\mathbf{y} | \mathbf{z})$, we could make the following assumption:

\begin{equation}
    q(\mathbf{y} |\mathbf{z}) \propto \operatorname{e} ^ { - (\mathbf{y} - \mathbf{o}) ^ { 2 }}, \; \text{where} \; \mathbf{o} = \operatorname{MLP}(\mathbf{z})
\end{equation}

\noindent which is equivalent to assume the output $\mathbf{o} $ of the MLP specifies the mean of a unit-variance Gaussian distribution. Under such assumption, the $\log q(\mathbf{y} | \mathbf{z})$ term in Eq. \ref{EQ:IBLBO} can be rewritten into the following form: 

\begin{equation}
\log q(\mathbf{y} | \mathbf{z}) = \sum - (\mathbf{y}-\mathbf{o}) ^ {2} + C
\end{equation}

\noindent where $C$ is a constant, and the maximization of which is equivalent to minimizing the mean square error (MSE) loss commonly used in the regression task.

\subsubsection{Temporal Regression: RNN with MSE Loss}

Moreover, if the target is a popularity sequence observed at interval after the post of a micro-video, the tractable family $\mathbf{Q}$ of the variational decoder $q(\mathbf{y} | \mathbf{z})$ can be chosen as the recurrent neural network (RNN). RNN captures the dynamic patterns of a sequence through the maintenance of a hidden state, whose value in one timestep is composed of information passed on from previous timestep and incorporation of new information from the outside. RNN is comprised of two parts: (1) a temporal dynamic that determines the internal evolution of hidden state; (2) a mapping from the RNN state to the output. Mathematically, the two parts of an RNN can be formulated as: 

\begin{equation}
\begin{aligned}
\mathbf{h} _ { t } &= \operatorname{H}(\mathbf{x} _ {t}, \mathbf{h} _ {t-1}) = \operatorname{Nonlinear}(\mathbf{W} _ { hh } \mathbf{h} _ {t-1} + \mathbf{W} _ { xh } \mathbf{x} _ {t})\\
\mathbf{o} _ { t } &= \operatorname{O}(\mathbf{h} _ {t})=\operatorname{Nonlinear}(\mathbf{W} _ { ho } \mathbf{h} _ {t})\\
\end{aligned}
\end{equation}

\noindent where $\mathbf{W} _ { xh }, \mathbf{W} _ { hh }, \mathbf{W} _ { ho }$ are the input-state, state-state, state-output transition matrices respectively. For vanilla RNN, $\mathbf{z}$ could be directly taken as its initial state $\mathbf{h} _ { 0 }$. However, if long-short term memory (LSTM) \cite{DBLP:journals/neco/HochreiterS97}, which addresses the exploding gradient issues of vanilla RNNs by the introduction of gating mechanism, is used, $\mathbf{z}$ could be transformed through two distinctive dense layers to get its initial context variable and hidden state separately. The input $\mathbf{x} _ {t}$ at each timestep is set to be the standardized absolute time of that prediction timestep (which is different from the RNN timestep that denotes the relative time to the start) based on the observation that the popularity level of a micro-video tends to vary in a regular fashion at different times of the day. Besides, in order to avoid the error accumulation due to the exposure bias \cite{DBLP:journals/corr/RanzatoCAZ15}, predicted results from the previous timestep is not fed into the RNN cell at the next timestep as additional inputs. 

Similar to the regression task, we can assume that the output of the RNN specifies the mean of a unit-variance Gaussian variable for the popularity level at each timestep:

\begin{equation}
    \label{EQ:TEMP_GAUSS}
    q(\mathbf{y} _ {t}|\mathbf{x}_{t}, \mathbf{h} _ {t-1}) \propto \operatorname{e} ^ { - (\mathbf{y} _ {t}-\mathbf{o} _ {t}) ^ { 2 }} \, \text{where} \, \mathbf{o} _ {t} = \operatorname{O}\left(\operatorname{H}\left(\mathbf{x}_{t}, \mathbf{h} _ {t-1}\right)\right)
\end{equation}

\noindent Then, utilizing the local Markov assumption of the RNNs \cite{DBLP:conf/nips/BengioDV00}, i.e., $(\mathbf{o}_{t} \perp \mathbf{o}_{t-1}, \mathbf{h}_{t-2} \, | \, \mathbf{h}_{t-1})$, the joint probability of the whole popularity sequence $\mathbf{y}$ can be factorized into the product of per timestep conditional probabilities (Eq. 
\ref{EQ:TEMP_GAUSS}). Therefore, the following equation holds:

\begin{equation}
\label{EQ:DEC_PROB}
\log q(\mathbf{y} | \mathbf{z}) = \sum  \log q(\mathbf{y} _ {t}|\mathbf{x}_{t}, \mathbf{h} _ {t-1}), \; \text{where} \; \mathbf{h} _ {0} = \operatorname {func}(\mathbf{z})
\end{equation}

\noindent Combining the Eq. \ref{EQ:TEMP_GAUSS} and \ref{EQ:DEC_PROB}, the $\log q(\mathbf{y} | \mathbf{z})$ can be expressed with the following form:

\begin{equation}
\log q(\mathbf{y} | \mathbf{z}) = \sum - (\mathbf{y} _ {t}-\mathbf{o} _ {t}) ^ {2} + C
\end{equation}

\noindent and the maximization of which is equivalent to the minimization of the sum of the per-step MSE losses.

\subsubsection{Summary of the MMVED framework}
By now, each component of the MMVED framework is fully specified, and we distinguish the MMVED for classification, regression, and temporal regression with MMVED-CLS, MMVED-REG, MMVED-TMP for the rest of the paper. Figure \ref{FIG:MODEL} schematically illustrates the structures of the  proposed MMVED framework.

\subsection{Monte Carlo Gradient Estimator}

The above sections describe in detail each component to compute the Eq. \ref{EQ:IBLBO}. However, the optimization of Eq. \ref{EQ:IBLBO} w.r.t the encoder $p(\mathbf{z}|\mathbf{m})$ and decoder $q(\mathbf{y}|\mathbf{z})$ is not trivial, since their gradients is in essence the gradients with probabilistic distributions, which precludes us from calculating them analytically. As a result, Monte Carlo (MC) methods are introduced to form unbiased estimators for the gradients. 

We assume that the encoder and the decoder network is parameterized by $\boldsymbol{\phi}$ and $\boldsymbol{\theta}$ respectively. Since the KL term in the IBLBO can be calculated analytically, the only gradient needs to be computed by sampling is the expected log likelihood term $\operatorname{
ELL} \triangleq  \mathbb { E } _ {p(\mathbf{z}|\mathbf{m})} \left[\log q (\mathbf{y}|\mathbf{z})\right]$. Observing that  the $\log q (\mathbf{y}|\mathbf{z})$ term is explicit in expectation form, its gradient w.r.t the ELL can be easily estimated with MC by averaging the gradients computed from samples taken from the encoder distribution:

\begin{equation}
\label{EQ:DEC_GRAD_EST}
\nabla _ {\boldsymbol{\theta}} \operatorname{
ELL} \simeq \frac { 1 } { L } \sum _ { l = 1 } ^ { L } \nabla _ {\boldsymbol{\theta}} \left( \log q _ { \boldsymbol{\theta}} \left( \mathbf{y}  | \mathbf{\mathring{z}}  ^ { (l) }\right) \right)
\end{equation}

\noindent where $\simeq$ symbol means that the RHS is an unbiased estimator for the LHS, and $\mathbf{\mathring{z}}  ^ { (l) }$ is the $l$th sample draw from the encoder distribution $p(\mathbf{z}|\mathbf{m})$. The unbiased estimator for the gradients w.r.t $\boldsymbol{\phi}$ is more difficult to obtain the since there is no obvious way to explicitly reform it into an expectation form. However, with reparameterization trick \cite{DBLP:journals/corr/KingmaW13}, samples from $p(\mathbf{z}|\mathbf{m})$ can be reformed by a bivariate transformation from a deterministic part and a stochastic noise of the following form:

\begin{equation}
\label{EQ:REP_TRICK}
\mathbf{\mathring{z}}\left([\boldsymbol{\mu}, \boldsymbol{\sigma}], \mathring{\boldsymbol{\epsilon}}\right) = \boldsymbol{\mu} + \boldsymbol{\sigma} \odot \mathring{\boldsymbol{\epsilon}}
\end{equation}

\noindent where $\mathring{\boldsymbol{\epsilon}}$ is a random vector the same size as $\mathbf{z}$ drawn from the standard normal distribution. Eq. \ref{EQ:REP_TRICK} means that the gradients of ELL w.r.t $\boldsymbol{\phi}$ can be estimated by drawing samples from the noise distribution taking and average:

\begin{equation}
\label{EQ:ENC_GRAD_EST}
\nabla _ {\boldsymbol{\phi}} \operatorname{
ELL} \simeq \frac { 1 } { L } \sum _ { l = 1 } ^ { L } \nabla _ {\mathbf{\mathring{z}} ^ {(l)}} \operatorname{ELL} \cdot \nabla _ {[\boldsymbol{\mu}, \boldsymbol{\sigma}]} \mathring{\mathbf{z}} ^ {(l)} \left([\boldsymbol{\mu}, \boldsymbol{\sigma}],  \mathring{\boldsymbol{\epsilon}} ^ {(l)}\right) \cdot \nabla _ {\boldsymbol{\phi}} {[\boldsymbol{\mu}, \boldsymbol{\sigma}]}
\end{equation}

As was discussed in \cite{gal2016uncertainty}, the variance of the reparameterization trick is low enough such that one sample is enough for the training to converge. Thus, the $L$ in both Eq. \ref{EQ:ENC_GRAD_EST} and Eq. \ref{EQ:DEC_GRAD_EST} is set to one. \textbf{Eq. \ref{EQ:REP_TRICK} can be viewed alternatively from the denoising perspective. It shows that during training, the hidden representations are injected with self-adaptive random Gaussian noises controlled by the standard deviation $\boldsymbol{\sigma}$, which was shown by previous work to have robustness towards noise in the test phase compared to adding a deterministic noise in feature space \cite{DBLP:conf/kdd/LiS17}, and provides a systematic way to model the latent noise generation mechanism incurred by the low-quality of micro-videos}. The specific training steps of our proposed MMVED is summarized in Algorithm \ref{alg:MMVED-SGD}.

\begin{algorithm}
    \let\OldRequire\algorithmirequire
	\renewcommand{\algorithmicrequire}{\textbf{Input:}} 
	\renewcommand{\algorithmicensure}{\textbf{Output:}}
	\caption{ MMVED-SGD: Training MMVED with SGD.}
	\label{alg:MMVED-SGD} 
	\begin{algorithmic}[1]
		\REQUIRE A micro-video popularity prediction dataset $\mathcal{ D }$; each sample is associated with $K$ modalities $\mathbf{m}$, and the corresponding popularity representation $\mathbf{y}$. \\
		
		\STATE Select proper decoder structure according to the task.
		\STATE Randomly initialize $\boldsymbol{\theta}$, $\boldsymbol{\phi}$.
		
		\WHILE{not converge}
		\STATE Randomly sample a batch $\mathcal{ \hat{ D } }$ from $\mathcal{ D }$.

		\FOR{$i = 1 \to K$}
		\STATE Compute $\boldsymbol{\mu} _ {i}$, $\boldsymbol{\sigma} _ {i}$ as Eq. \ref{EQ:MOD_SPE_GAUSSIAN}.
		\ENDFOR
		\STATE Compute $\boldsymbol{\mu}$, $\boldsymbol{\sigma}$ of $p(\mathbf{z}|\mathbf{m})$ from each $\boldsymbol{\mu} _ {i}$, $\boldsymbol{\sigma} _ {i}$ as Eq. \ref{EQ:WHOLE_GAUSSIAN}, and add its KL-divergence with prior to the loss.
		\STATE Sample $\epsilon$ from $\mathbf{z}$ and compute $z$ with reparameterization trick as Eq. \ref{EQ:REP_TRICK}.
		\STATE Compute the $q(\mathbf{y}|\mathbf{z})$ as Eq. \ref{EQ:DEC_PROB}, and add prediction loss.
		\STATE Compute gradient of loss w.r.t $\boldsymbol{\theta}$, $\boldsymbol{\phi}$ with Eq. \ref{EQ:DEC_GRAD_EST}, \ref{EQ:ENC_GRAD_EST}.
		\STATE Update $\boldsymbol{\theta}$, $\boldsymbol{\phi}$
		\ENDWHILE
		\RETURN $\boldsymbol{\theta}$, $\boldsymbol{\phi}$
		
		\ENSURE MMVED model trained on dataset $\mathcal{D}$
	\end{algorithmic} 
\end{algorithm}

\subsection{Predictive Behaviors of MMVED}

Although MMVED appears a certain amount of randomness to deal with the interference of uncertainty in popularity during the training phase, its behavior in validation and prediction is designed to be deterministic, such that the results of different rounds of prediction for the same micro-video are consistent. Since the mean of the Gaussian variable carries the information whereas the standard deviation preserves the uncertainty, after the training of MMVED, for predicting the popularity of a newly released micro-video, the mean of the Gaussian embedding is kept as its fixed representation. Then, the extracted representation is fed into the decoder to make the prediction of the popularity.

\section{Feature Engineering}

\label{SEC:FEATURE}

In this section, we discuss the form of $\mathbf{m}$, i.e., the multimodal representation of the micro-video. Previous work has confirmed that whether a micro-video comes into fashion after its release is closely related to its visual-aural contents, its attached descriptions, and the profile of its publisher \cite{chen2016micro}. Thus, features are extracted from four different aspects, namely, visual, acoustic, textual, and social, as the multimodal characterization of the micro-video.

\subsection{Visual Modality}

In order to describe the visual contents of the micro-video, we utilize the state-of-the-art convolutional neural network (CNN). Specifically, we keep the convolutional base of the ResNet50 \cite{DBLP:conf/cvpr/HeZRS16} pre-trained on ImageNet \cite{DBLP:journals/ijcv/RussakovskyDSKS15} as the fixed feature extractor, which has achieved great success in various computer vision tasks such as image classification, object detection, and action recognition. 

As is shown in \cite{DBLP:conf/eccv/ZeilerF14}, the last layers of CNNs pre-trained on ImageNet encode the information with regard to the existence of certain objects and their relationships, which we believe could be of benefit to the popularity prediction of micro-videos, since certain objects, such as delicious food and cute pets are naturally more attractive than others. The detailed feature extraction steps are described as follows: for each micro-video, we first extract its keyframes with the FFmpeg toolkit\footnote{\url{http://ffmpeg.org/.}}. Then, for each keyframe, a 2,048-D activation is obtained from the global average pooling layer of the pre-trained ResNet. In the end, the extracted activations are temporally averaged to get a fixed-length representation for the whole micro-video.

\subsection{Acoustic Modality}

The audio modality of a micro-video usually includes two key types of information that could be helpful to predict its popularity: the original audio track that provides complementary information with the visual elements, such as the tone and speech of the protagonists, and the accompanied background music, which shows a strong hint for the affective states of the uploader. Similarly, we use the latest CNN that is specifically designed for audio spectrogram to extract the aural features, which has shown promising results in micro-video recommendation \cite{wei2019mmgcn}, hashtag recommendation \cite{wei2019personalized} and video affective content analysis \cite{zhu2019multimodal}.

In particular, for each micro-video with $N_{v}$ keyframes, we first split the audio track of length $N_{a}$ into $N_{v}$ non-overlapping one-second windows, with the center of the $i$th audio window $i _ {a}$ computed as follows:

\begin{equation}
i _ {a} = \frac{N _ {a} - invSR}{N _ {v}} i  + 0.5 invSR 
\end{equation}

\noindent where $invSR$ is the inverse of the audio sample rate. Under such cases, the center of each audio window precisely matches the position of the corresponding keyframe. Then, the spectrogram for each audio window is computed with code open-sourced by \cite{Hershey2016CNN}, and is fed into VGGish \cite{Hershey2016CNN} pre-trained on AudioSet \cite{gemmeke2017audio} to extract a 128-dimensional deep feature. The features for the same micro-video are globally pooled, such that video-level aural features are obtained.

\subsection{Textual Modality}

Besides, micro-videos are usually accompanied with textual descriptions, such as the titles, remarks, or some hashtags, to summarize the important content and emotional feelings the uploaders wish to convey, or accentuate the salient characteristics of the micro-video. Thus, the attached texts provide new aspects to support the popularity prediction. For example, a title containing "funny moments of Husky!" suggests that the micro-video is intended to resonate with a large population of pet lovers, whereas the hashtags like "AdvancedTheoryDiscovery" indicate the post videos are specially made for a small group. 

For the majority of the languages that are used worldwide, mature natural language processing (NLP) toolkits are usually developed and made public by the NLP specialists. For example, the FudanNLP\footnote{\url{https://github.com/FudanNLP/nlp-beginner/.}} \cite{DBLP:conf/acl/QiuZH13} toolkit, which is one of the most commonly referred arsenal for Chinese natural language processing, can make classification of all the texts based on models pre-trained on large corpora, and extract a 20-D textual feature vector for each sentence. As for English, the sentence2vec\footnote{\url{https://github.com/klb3713/sentence2vec/.}} \cite{le2014distributed} is able to embed English sentences of variable-length into a hidden space based on their semantic similarities. Besides, the Stanford CoreNLP toolkits\footnote{\url{http://stanfordnlp.github.io/CoreNLP/.}} \cite{manning-EtAl} is equipped with a broad range of text analysis packages, one of which is able to classify the input texts into five sentiment classes with pre-trained sentiment TreeBank model \cite{socher2013recursive}. Both toolkits were shown to be handy for extracting reliable textual features that can support the popularity prediction of the micro-video \cite{chen2016micro, DBLP:journals/tkde/JingSNBLW18}. 

\subsection{Social Modality}

In addition to the features that are directly linked to the micro-video content, the profile of its publisher could also provide key information for predicting its potential popularity. For instance, uploaders who have more followers or have their accounts verified may have more influence, and their productions tend to attract more attention among viewers, in comparison with the common users. 

As different micro-video sharing platforms unveil different attributes of the uploaders to the public, features from the social modality may vary from platform to platform. Nonetheless, we summarize some useful and universal social characteristics specific to the publisher: the follower-followee counts, the total post counts, the loop counts, the verification status, etc., which mainly portray the degree of influence of the uploader.

\subsection{Postprocessing}

Finally, all features extracted in four modalities are standardized into zero-mean and unit variance to eliminate the scale bias and stabilize the training of MMVED network.

\section{Experiments on the Public NUS datasets}
\label{SEC:NUS}

\subsection{Dataset and Implementation Details}

We first focus on the micro-video popularity regression task with MMVED-REG model on NUS dataset. The NUS micro-video popularity regression dataset\footnote{\url{http://acmmm2016.wixsite.com/micro-videos/.}} is established and released by researchers from the Lab for Media Search (LMS) at the National University of Singapore. The dataset contains 303,242 micro-videos collected between July 2015 and October 2015 from a then widely-used micro-video sharing platform Vine\footnote{\url{https://vine.co/.}}, most of which last 6-8 seconds. The stabilized value of the comments, reposts, likes, and loops number, are recorded and averaged to formulate the sole popularity score of a micro-video. Unfortunately, at the time of our experiments, a proportion of the links to micro-videos in the NUS dataset were invalid, and we could successfully download only 186,637 of them. Therefore, for a fair comparison, we keep the same number of test samples with previous papers (303,24, 10\% of 303,242) \cite{chen2016micro, DBLP:journals/tkde/JingSNBLW18}, and put aside another 303,24 samples for validation. The models are tested over five random splits of the dataset, and the averaged prediction performances are reported.

In our implementation of MMVED-REG model, for visual, acoustic, textual, and social modality, the number of units for the hidden layers of the modality-specific MLP encoder is empirically set to 32, 8+8 (for mean and logstd respectively).  We use Adam \cite{DBLP:journals/corr/KingmaB14} as the SGD optimizer, with learning rate initialized at $5e ^ {-4}$ and linearly dropped to  $5e ^ {-5}$ at the end. Training stops after 50 epochs. As the regularization coefficient $\lambda$ is very important to the performance of our framework, we will first discuss its impact in section \ref{SEC:PA_REG}, and then set it fixed to the optimal empirical value based on the experimental results in sections after.

\subsection{Model Evaluation}

\subsubsection{Evaluation Metric}

For our method to be comparable with previous work, we follow the usage of normalized mean squared error (nMSE) first proposed in \cite{nie2015beyond} to measure the performance of our model. The nMSE metric is defined as follows:

\begin{equation}
\operatorname{nMSE}= \frac{1}{N \sigma ^{2}} \sum_{i=1}^{N} (y _{i} - \hat{y} _ {i}) ^ {2}
\end{equation}

\noindent where $y _ {i}$ and $\hat{y} _ {i} \in \mathbb{R}$ are the real and predicted popularity score for the $i$th micro-video sample, and $\sigma$ is the standard deviation of the popularity groundtruth. Intuitively, the nMSE re-scales the normal MSE metric with the groundtruth variation, which eliminates the bias incurred by varied groundtruth variance for different dataset splits.

\subsubsection{Parameter Analysis}
\label{SEC:PA_REG}

As is aforementioned, the role $\lambda$ plays can be viewed from two aspects: First, Eq. \ref{EQ:UCOBJ} shows that $\lambda$ controls the penalty for encoding extra unit of information of the input modalities $\mathbf{m}$ into the hidden representation $\mathbf{z}$; Second, Eq. \ref{EQ:IBLBO} shows that $\lambda$ is the weight for the KL term and thus controls the penalty for deviation of the encoder distribution from the standard Gaussian prior. Both views can be unified in the sense that the closer the encoder distribution is to the uninformative prior, the less information in $\mathbf{m}$ would be left in $\mathbf{z}$.

In order to find the optimal value for $\lambda$, we train our MMVED-REG with different settings of $\lambda$, record their training dynamics, and evaluate their performance on the test set. The results are illustrated in Figure \ref{FIG:SENSI_ANA_REG}. As we can find out from Figure \ref{FIG:SENSI_ANA_REG},  generally, the test performance increases first and then drops with the increment of $\lambda$. Such phenomenon could be explained by the fact that when the value of $\lambda$ is too small, the information bottleneck constraint set on the Gaussian representation is relaxed, and thus information flows unbridledly from the noisy input modalities to the hidden representations, which makes the training unstable, oscillating and hard to converge (the blue curve). On the other hand, when the value of $\lambda$ grows too large, excessive constraint is set upon the hidden embedding, where useful information for popularity prediction is blocked out as well as the noisy ones; therefore, the multimodal encoder is unable to learn a good representation from the micro-video content, which leads the model to converge to a suboptima (the red curve) compared to the model trained with a suitable $\lambda$ (the green curve). 

Besides, we notice an interesting phenomenon that for our model, making insufficient information constraint ($\lambda=0.1$) is worse than imposing oversized information constraint ($\lambda=1.1$). Such results could reveal that the information-rich but low-quality nature of micro-videos renders the extracted multimodal features low in  signal-to-noise ratio, which not only interferes with but sabotages the training process and generalization ability of the popularity prediction model. Therefore, based on the analysis above, we fix $\lambda$ of the MMVED-REG to its empirical optimal value 0.7, when we draw the comparison with other SoTA methods on the NUS dataset in the next section.

\begin{figure}
\subfigure[nMSE]{
\begin{minipage}[t]{0.45\linewidth}
\centering
\includegraphics[scale=0.4]{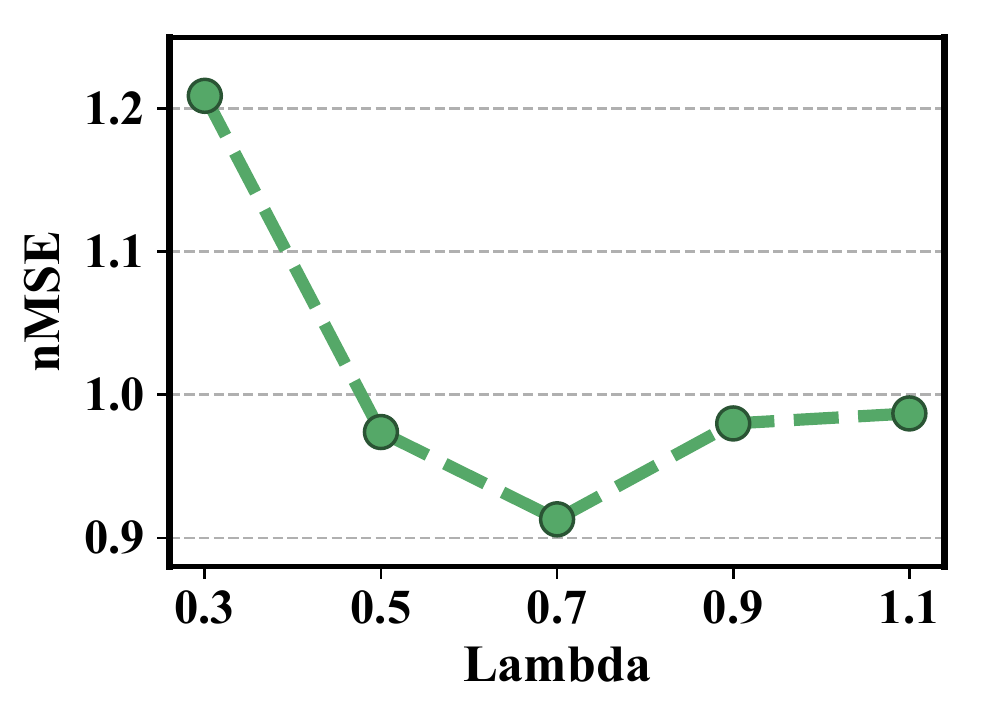}
\end{minipage}
}
\subfigure[Training dynamic]{
\begin{minipage}[t]{0.45\linewidth}
\centering
\includegraphics[scale=0.4]{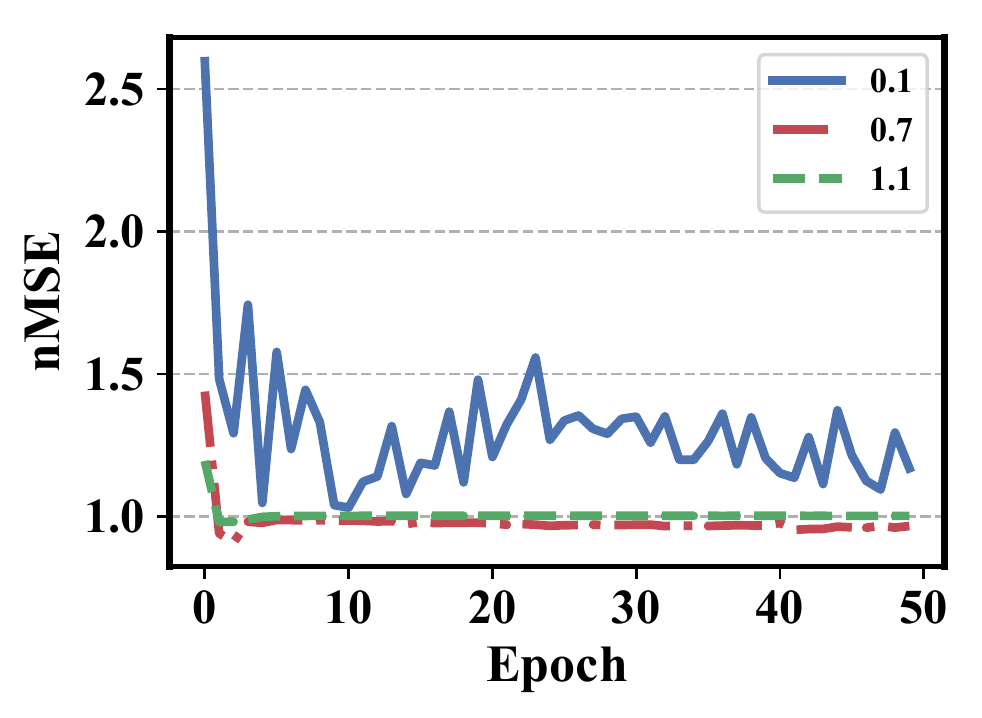}
\end{minipage}
}
\caption{The influence of $\lambda$ to model performance and training dynamic}
\label{FIG:SENSI_ANA_REG}
\end{figure}

\subsubsection{Comparison with the State-of-the-Arts}

The SoTA methods that are selected for the comparison with the proposed MMVED-REG model are listed below:

\begin{itemize}
\item
    \textbf{ELM}. The extreme learning machine \cite{huang2006extreme}, \cite{huang2011extreme} is a widely used regression model that is comprised of a single hidden layer neural network. The ELM is shown to be able to learn the embedding of various types of feature while having high generalization and low computational costs.
    
\item
    \textbf{TMALL}. The transductive multi-modal learning model \cite{chen2016micro} regresses the popularity of micro-videos through learning embeddings of micro-videos where the hidden space is constrained to be consistent with the multimodal feature spaces, such that the semantic characteristics of micro-videos are preserved.
    
\item
    \textbf{TRLMVR}. The transductive low-rank multi-view regression model \cite{DBLP:journals/tkde/JingSNBLW18} is an extension of TMALL, where a novel low-rank constraint is set upon the learned micro-video embeddings such that only principal components of the feature space are allowed to be kept in the representations.
\end{itemize}

\begin{table}[]
\centering
\caption{Comparison of performance between the proposed MMVED-REG model and several SoTA methods.}
\label{TAB:NUS_COM_BASE}
\begin{tabular}{cc}
\toprule
Methods & nMSE\\ 
\midrule
ELM \cite{huang2011extreme}                   & 0.982   \\
TMALL \cite{chen2016micro}                    & 0.979   \\
TLRMVR \cite{DBLP:journals/tkde/JingSNBLW18}  & 0.934   \\
\midrule
The proposed MMVED-REG       & \textbf{0.914} \\ 
\bottomrule
\end{tabular}
\begin{tablenotes}
 \item[1] The best results are in \textbf{bold}. 
\end{tablenotes}
\end{table}

Table \ref{TAB:NUS_COM_BASE} reports the performances of the proposed MMVED-REG and the SoTA algorithms. From Table \ref{TAB:NUS_COM_BASE}, we first notice that compared to ELM where features from different modalities are simply concatenated, TMALL utilizes a multi-view method to fuse heterogeneous features from four modalities subject to a consistency constraint, and it achieves a better result than ELM. TRLMVR further improves TMALL by adding a low-rank constraint of hidden space to the multi-view learning objective such that the insignificant components of the feature spaces are removed, which leads to substantial improvements compared to TMALL. However, in TRLMVR, the low-rank constraint is hardwired and universal for all samples, such that samples with different noise levels could not be properly distinguished. Besides, their hidden representation of TRLMVR is inherently deterministic. On the other hand, in our MMVED-REG model, we first preserve the uncertainty of popularity mapping by modeling the hidden representation as a Gaussian variable. In addition, for each micro-video, the mean (semantic part) and std (uncertainty and noise part) of its Gaussian embedding are independently inferred through the PoE inference network. Furthermore, an information bottleneck constraint is set upon the Gaussian embeddings, such that the relevant information obscured by the information-rich but noisy micro-video contents can be actively learned to be extracted into the hidden representation with the guidance of the training popularity groundtruth. Therefore, our method achieves the best results among the SoTA methods. 

\section{Experiments on Crawled Xigua Dataset}
\label{SEC:XIGUA}

\subsection{Dataset and Implementation Details}

Next, our discussion shifts towards the temporal popularity prediction of micro-videos on Xigua dataset. The Xigua micro-video dataset is established based on a famous micro-video sharing platform Xigua\footnote{\url{https://www.ixigua.com/}} in China, whose distinguishing characteristic is that the increments of loops for each micro-video are contiguously recorded every 15 minutes immediately after its release for 3 days, which makes the popularity groundtruth a sequence. The detailed crawling strategy and dataset description can be referred to in \cite{MMVED-WWW2020}. Xigua dataset includes 3,231,072 records from 11,219 micro-videos posted by 2664 users between July 24th and August 14th, 2019. In Xigua dataset, the ResNet visual feature is reduced to 128D by PCA to prevent overfitting, and the followers/followees counts of the uploader together with the verification status are recorded as the 3D social feature.

Since predicting popularity every 15-minute is too fine in granularity for practical use, the popularity sequences are re-sampled such that the recording interval $T_{int}$ is equivalent to 8 hours for the main part of the experiments. Nonetheless, the influence of $T_{int}$ will be thoroughly discussed in section \ref{SEC:LENGTH}. The reported performances, unless other-wisely specified, are the averaged results over five splits, each of which randomly selects 64\% of the micro-videos for training, 16\% for validation, and 20\% for testing. For the implementation of MMVED-TMP, the encoder structure, and the training strategy remain similar to MMVED-REG. Besides, the dimension of the hidden units for the variational RNN decoder is set to 8 empirically.

\subsection{Model Evaluation}

\begin{figure}
\subfigure[nMSE-TMP]{
\begin{minipage}[t]{0.45\linewidth}
\centering
\includegraphics[width=0.98\linewidth]{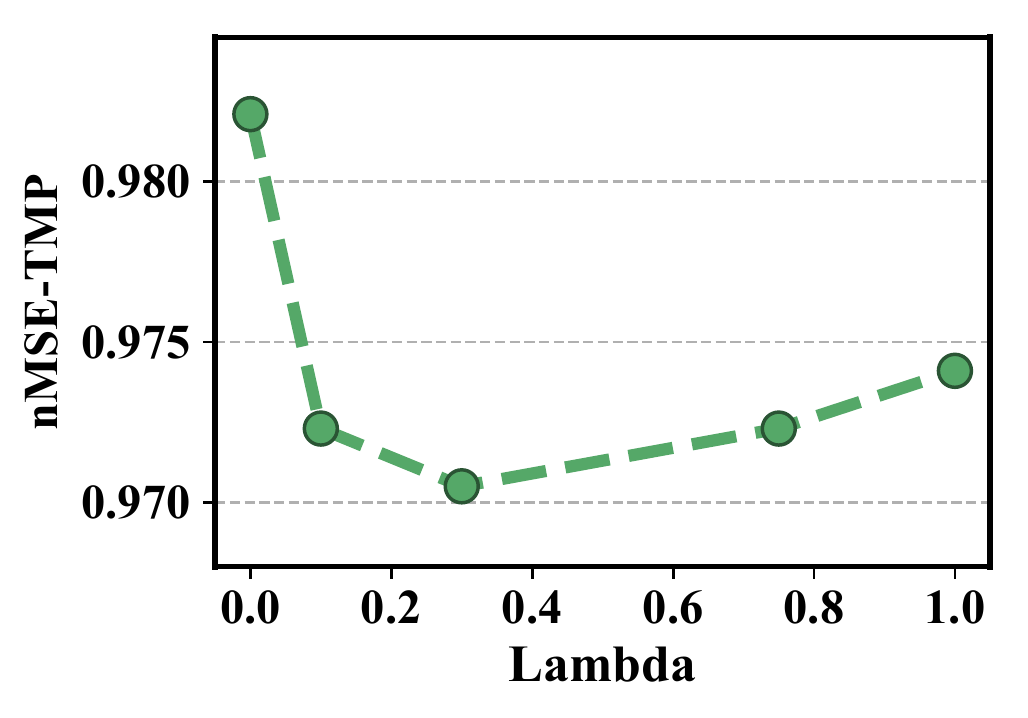}
\end{minipage}
}
\subfigure[SRC]{
\begin{minipage}[t]{0.45\linewidth}
\centering
\includegraphics[width=0.98\linewidth]{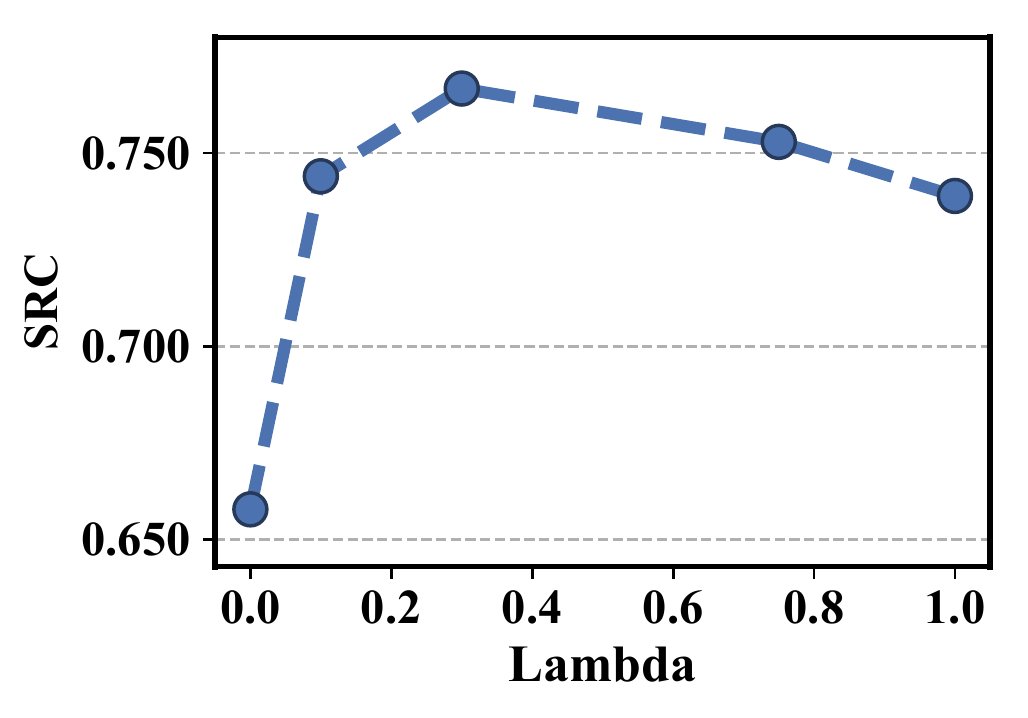}
\end{minipage}
}
\caption{Influence of $\lambda$ to  performance of proposed MMVED-TMP model under the nMSE and SRC metrics.}
\label{FIG:SENSI_ANA}
\end{figure}

\begin{figure}
\subfigure[nMSE-TMP]{
\begin{minipage}[t]{0.45\linewidth}
\centering
\includegraphics[width=0.98\linewidth]{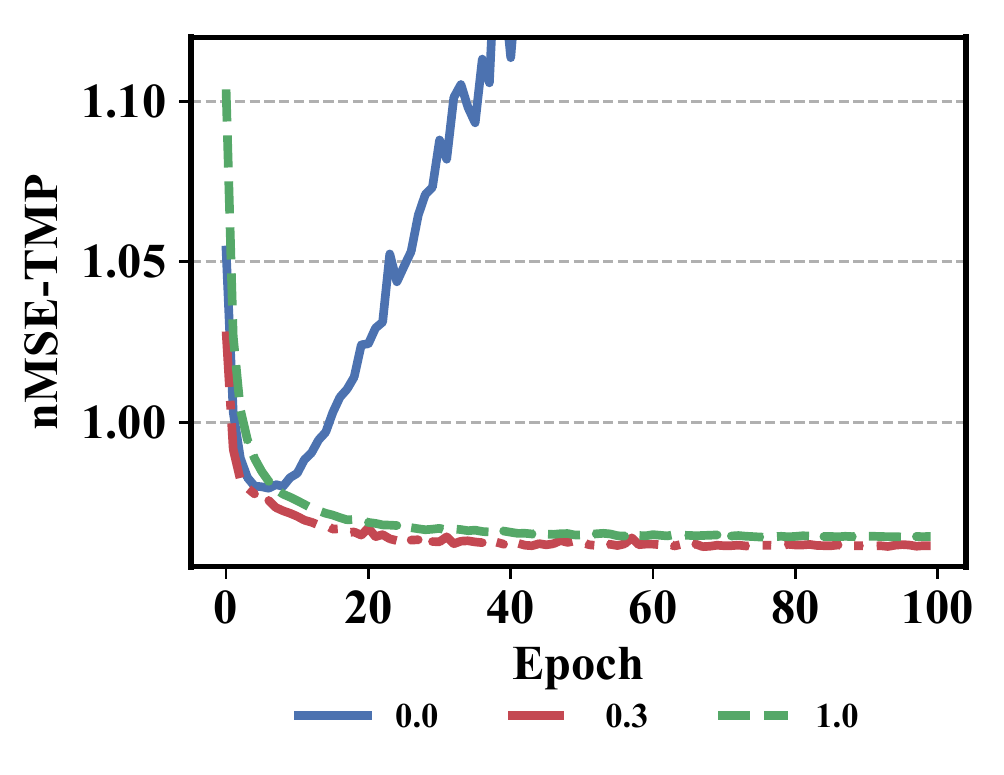}
\end{minipage}
}
\subfigure[SRC]{
\begin{minipage}[t]{0.45\linewidth}
\centering
\includegraphics[width=0.98\linewidth]{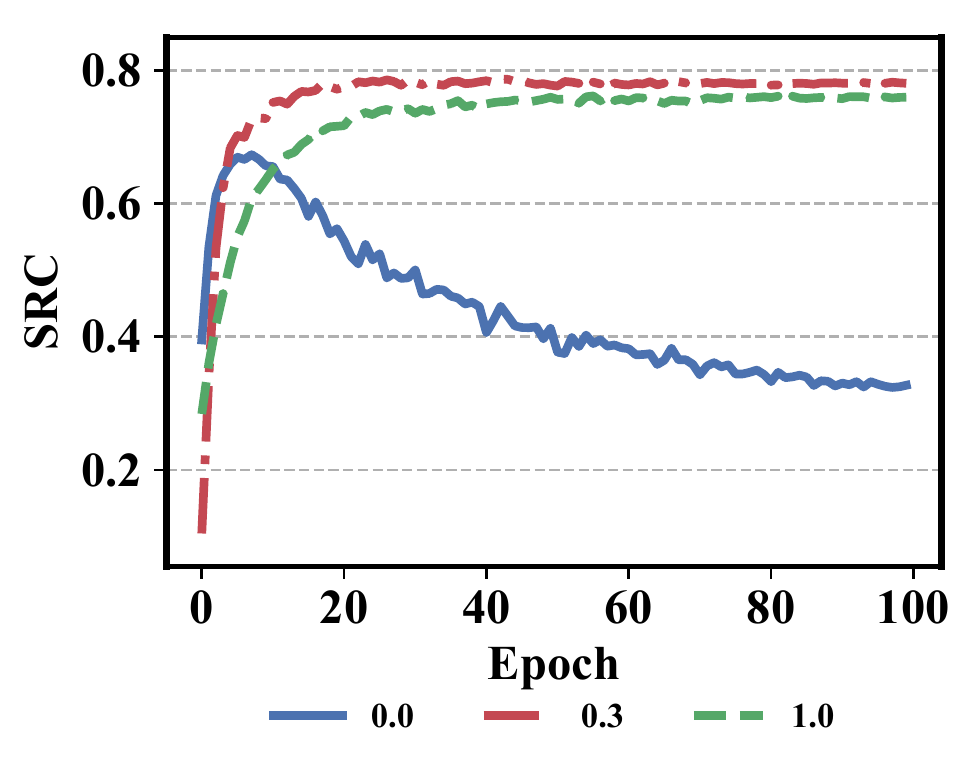}
\end{minipage}
}
\caption{Influence of $\lambda$ to training dynamics of proposed MMVED-TMP model under the nMSE and SRC metrics.}
\label{FIG:VAL_PER}
\end{figure}

\subsubsection{Evaluation Metrics}

Two metrics are utilized in our work to evaluate the performance of different temporal popularity prediction  methods. The first is the temporal normalized Mean Squared Error (nMSE-TMP), which is a variant of nMSE that is suitable for the measure of closeness between two sequences. The nMSE-TMP metric utilized in our paper is defined as:

\begin{equation}
\operatorname{nMSE-TMP}= \frac{1}{N\sigma _ {\mathbf{y}} ^ {2}} \sum_{i=1}^{N} \frac{1}{T} \sum_{j=1}^{T} (\mathbf{y} _{i} ^{j} - \hat{\mathbf{y}}_{i}^{j})
\end{equation}

\noindent where $\mathbf{y} _ {i}$ and $\hat{\mathbf{y}} _ {i} \in \mathbb{R}^{T}$ are the real and predicted popularity sequence for the $i$th micro-video sample in the dataset, and $\sigma$ stands for the standard deviation operator. The nMSE-TMP index rescales the mean squared error of predictions by the groundtruth variance, which alleviates the bias due to the difference of fluctuation among popularity sequences and the difference of variation among the dataset splits.

Besides, observing that even two sequences close enough measured by the nMSE-TMP could have the opposite trend, we adopt the Spearman’s rank correlation (SRC) metric \cite{DBLP:conf/www/KhoslaSH14} as a complement, which is defined as follows:

\begin{equation}
\operatorname{SRC}=\frac{1}{N} \sum _ {i=1} ^ {N} \frac{1}{T} \sum_{j=1}^{T}\left(\frac{\mathbf{y}_{i} ^ {j} -\overline{\mathbf{y}}_{i}}{\sigma_{\mathbf{y} _ {i}}}\right)\left(\frac{\hat{\mathbf{y}}_{i} ^ {j}-\overline{\hat{\mathbf{y}}}_{i}}{\sigma_{\hat{\mathbf{y}} _ {i}}}\right)
\end{equation}

\noindent Generally, SRC indicates the trend consistency of the prediction, which is complementary to the nMSE-TMP metric since the latter one only focuses on the deviation of squared absolute value between $\mathbf{y} _ {i}$ and $\hat{\mathbf{y}} _ {i}$.

\subsubsection{Parameter Analysis}
\label{SEC:PARAM_ANA}

Similarly, we first analyze the impact of $\lambda$ to the performance and training dynamic of MMVED-TMP model under the nMSE and SRC metric. The results are illustrated in Figure \ref{FIG:SENSI_ANA} and \ref{FIG:VAL_PER}. From both figures, we can find out that the relationship between $\lambda$, model performance, and training dynamic is in line with the regression task. Both metrics indicate that the model performance improves first but then worsens with the increment of $\lambda$. Besides, with $\lambda$ set to $0$, where no information constraint is set upon the hidden representation, the validation performance slightly increases in the first few epochs but drops drastically afterwards, which indicates that the model overfits to the noise and irrelevant information specific to the training set, whereas with $\lambda$ set to 0.3 or 1, since the information capacity of the embedding is restricted by the KL divergence constraint, the validation performance improves in a nearly monotonic manner. Moreover, with the $\lambda$ set to the suitable value ($\lambda$ = 0.3), the asymptotic optimal performance of MMVED-TMP is better than those with too large $\lambda$. 

Besides, we also note several differences between MMVED-TMP on the Xigua dataset and MMVED-REG on the NUS dataset when analyzing the sensitivity of $\lambda$. The most salient one is that for MMVED-TMP, the optimal value for $\lambda$ shifts to a smaller value 0.3, which implies that under our model design and experimental settings, less constraint leads to better performance for the temporal popularity regression task. Moreover, for MMVED-TMP model, if no constraint is set upon the hidden representation ($\lambda$=0), although over-fits quickly, which is reflected by the rapid rebound of validation loss, the training loss converges for all splits (while for MMVED-REG model, training fails to converge for some splits). Such phenomena could be explained as follows: First, in \cite{MMVED-WWW2020}, by visualizing the statistical characteristics of the popularity sequence in a stratified manner, we found that the popularity level of micro-videos tend to have similar temporal evolution patterns, reaching the peak shortly after their launch, and then gradually dropping to near-zero and stabilize, as people show preference to fresh micro-videos than the stale and obsolete ones. Thus, the absolute time which is fed into the RNN at each timestep is itself a strong indicator for popularity level, which makes MMVED-TMP less dependent on the latent representation compared to MMVED-REG, for which, on the contrary, the noisy micro-video content is the sole source of information. Therefore, the MMVED-TMP model exhibits more tolerance to the irrelevant information in the hidden representation. Since empirically, setting $\lambda$ to 0.3 reaches the optimal performance for the MMVED-TMP model, we use it as the default value for $\lambda$ in the sections afterwards.

\subsubsection{Comparison with the Baselines}

In order to further prove the effectiveness of the proposed MMVED-TMP model, in this section, we choose two machine learning baselines, i.e., the temporal support vector regressor and the temporal random forest, and two deep learning baselines, i.e., the contextual LSTM, the multimodal deterministic encoder-decoder, to verify the superiority of the proposed MMVED-TMP framework:

\begin{itemize}
\item
    \textbf{Temporal SVR}. The Support Vector Regressor (SVR) is a widely-used kernel method in regression areas. For SVR to be amenable to the temporal regression task, the absolute time is jointly concatenated with the features of micro-videos from four modalities.
    
\item
    \textbf{Temporal RFR}. The Random Forest Regressor (RFR) is another machine learning method based on an ensemble of decision trees. We make the same pre-processing for features as the Temporal SVR to make the RFR suitable for the temporal prediction task.
    
\item
    \textbf{Contextual LSTM}. The Contextual LSTM \cite{DBLP:journals/corr/GhoshVSRDH16} augments the input at each timestep of the LSTM with a global embedding of the object as auxiliary contextual information. In our implementation, we encode the concatenation of features from all modalities with an MLP as the contextual variable.
    
\item
    \textbf{Deterministic Encoder-Decoder}. In order to verify the effectiveness of uncertainty preservation, we design the multimodal deterministic encoder-decoder (MMDED) baseline, where the mean of the PoE Gaussian embedding is taken as a deterministic representation of the input modalities. All other structures remain the same as the MMVED-TMP.
\end{itemize}

\begin{table}[]
\centering
\caption{Comparison of performance between the proposed MMVED-TMP model and several baseline methods.}
\label{TAB:COM_BASE}
\begin{tabular}{ccc}
\toprule
Methods & nMSE-TMP & SRC\\ 
\midrule
Temporal SVR    & 1.018  & 0.247 \\
Temporal RFR    & 1.204  & 0.127 \\
CLSTM           & 0.987  & 0.634 \\
MMDED           & 0.979  & 0.655 \\
\midrule
The proposed MMVED-TMP  & \textbf{0.971}  & \textbf{0.767} \\ 
\bottomrule
\end{tabular}

\begin{tablenotes}
 \item[1] The best results are in bold. 
 \item[2] Small nMSE and large SRC indicate good performance.
\end{tablenotes}

\end{table}

Table \ref{TAB:COM_BASE} summarizes the comparison of performance between
the proposed MMVED-TMP model and the baselines. Three conclusions can be drawn from the Table \ref{TAB:COM_BASE}: First, deep learning-based methods consistently outperform the machine learning-based ones by a large margin, since for the machine learning-based baselines, the prediction of popularity level at each timestep is made independently by taking the absolute time as an auxiliary 1D-feature, and thus no temporal relationship is utilized in these methods, whereas capturing the dynamic pattern of sequence, on the contrary, is the strength of the RNNs. Second, for the encode-decoder based deep learning methods, those that take the hidden encoding as the initial state of the RNN (MMDED, MMVED-TMP) outperform those use hidden encoding to augment the input at each timestep (CLSTM). The reason could be that by only using the embedding micro-video as an initialization for the hidden state of the RNN, the former methods force the encoder to summarize relevant information from the micro-video features with respect to the popularity trend, blocking the undesired shortcut for the network to base the prediction solely on the absolute time and ignore the encoding. Besides, by concatenating the high dimensional micro-video feature with the absolute time, the CLSTM model dilates the 1D absolute time information, which is shown to be the strong indicator for the popularity trend of the micro-videos in the previous section.

Finally, we also confirm that modeling the hidden representation as a stochastic variable (MMVED-TMP) performs better than its deterministic counterpart (MMDED). Two reasons could explain such result: first, the mapping from the feature of micro-video to its popularity trend is non-deterministic as external uncertain factors also influence its popularity level, and encoding the micro-video content into stochastic Gaussian embedding where the uncertainty is preserved in its variance is flexible enough to deal with such randomness. Besides, the MMVED-TMP is able to systematically model the hidden noise generative process by adding a self-adaptive noise to its encoded hidden representation during the training phase, which shows more robustness to noise than the MMDED model when faced with new data.

\subsubsection{Influence of Sampling Granularity}
\label{SEC:LENGTH}

\begin{figure}
\centering
\includegraphics[width=0.85\linewidth]{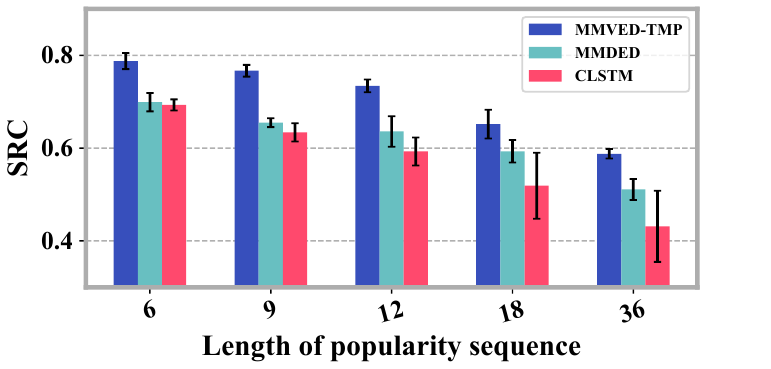}
\caption{The influence of sampling granularity (sequence length) to performance of proposed MMVED-TMP and two RNN-based baselines}
\label{FIG:INFULENCE_LENGTH}
\end{figure}

In order to gain more insight into the effects of the popularity sequence length to the performance of the proposed MMVED-TMP model and the baselines, in this section we test the model performance with varied re-sample interval $T_{int}$ (which, as a result, changes the popularity sequence length). As the change of $T_{int}$ alters the groundtruth variance, which invalidates the comparison of nMSE metric, we only use the SRC metric as the performance indicator. The evaluation results are illustrated in Figure \ref{FIG:INFULENCE_LENGTH}. Figure \ref{FIG:INFULENCE_LENGTH} shows that generally, for all three models, SRC drops as the sequence length increases, since longer sequence makes it harder to propagate gradient information at the back of the popularity sequence to both the temporal RNN and the multimodal encoder; Besides, the prediction error could accumulate more severely with longer output sequences. However, even for popularity sequence of length 36, which is quite large due to the off-line nature of our task, the prediction made by MMVED-TMP still gets more than 0.5 of correlation with the groundtruth and shows the lowest variance of all dataset splits compared with two other baselines, which demonstrates the potential generalization ability of the proposed MMVED-TMP model to popularity prediction tasks with finer granularity and longer term.

\subsubsection{Robustness to Modality Missing in Test Phase}
\label{SEc:ROBUST}
\begin{table}[]
\centering
\caption{Influence of modality missing problems to MMVED-TMP model under nMSE and SRC criteria.}
\label{TAB:DIFF_MOD}
\begin{tabular}{cccc}
\toprule
Available Modalities & nMSE-TMP & SRC\\ 
\midrule
A + T + S    & 0.9706  & 0.7669\\
V + T + S    & 0.9707  & 0.7670\\
V + A + S    & 0.9705  & 0.7669\\
V + T + A    & 0.9710  & 0.7614\\
\midrule
No missing modality     & \textbf{0.9705}  & \textbf{0.7667} \\ 
\bottomrule
\end{tabular}
\end{table}

Finally, in this section, we demonstrate the robustness of our model under the circumstances where some modalities are missing in the test phase. Such a problem is quite common since some users are naturally reluctant to turn on the microphone when film their videos, to add descriptions to their posted micro-videos, or to upload any user profile. As is mentioned in section \ref{SEC:MMENCODER}, with the conditional independence assumption, when faced with new samples with fewer modalities than those in the training set, we simply drop the corresponding encoder networks and re-compute the product-of-experts distribution with the encodings calculated from the available modalities, saving the effort to re-train the whole MMVED framework from scratch. 

In the experiments, we fix the weights of the trained MMVED-TMP model, eliminate one modality with the corresponding encoder network at a time, and report the test performance. The results are listed in Table \ref{TAB:DIFF_MOD}. From Table \ref{TAB:DIFF_MOD}, we can find out that the detrimental effect of missing one modality is quite negligible. Aside from the strong temporal features, such results could also be explained by the PoE nature of the multimodal encoder, where the decision for which information to encode from the micro-video content into the hidden representation is made based on the complementary information from all modalities, weighted by their confidence. We observe that the elimination of social modality costs the performance the most, which is in agreement with \cite{DBLP:journals/tkde/JingSNBLW18}, accentuating the important role the uploader plays for the potential popularity of a micro-video. The elimination of visual and aural modality, on the other hand, shows least effects on the performance, since although features from both modalities have the largest dimension, they are redundant and can be complemented by information from other available modalities, whereas the user attributes from the social modality are compact and comparatively irreplaceable.

\section{Conclusions}
\label{SEC:CONCLUSION}

In this paper, we propose a multimodal variational encoder-decoder (MMVED) framework for micro-video popularity prediction tasks. The proposed MMVED is capable of learning a stochastic Gaussian embedding of a micro-video that is informative to its potential popularity level, where the inherent uncertainty is properly preserved. Besides, faced with the internal noise of micro-videos, a deep information bottleneck constraint is set upon the Gaussian embedding such that only relevant information is allowed to flow from the input modalities to the hidden representations.

Based on experiments on two real-world datasets, we draw the following conclusions: First, explicitly taking the uncertainty into consideration by modeling the hidden representation as random variables could improve the performance of micro-video popularity prediction models compared with their deterministic counterpart. Second, it is necessary to make constraints on the information contained in the stochastic embedding to deal with the noise in micro-video features, such trained models can achieve satisfactory generalization abilities. 

\bibliography{ms}

\newpage
\appendix
\appendixpage

\section{The Deep Variational Information Bottleneck Lower Bound for MMVED}
\label{APP:DIB}

\newtheorem{theorem}{Theorem}
\begin{theorem}{}
\label{THEO_1}
If $p(\mathbf{z} | \mathbf{m}, \mathbf{y})=p(\mathbf{z} | \mathbf{m})$ holds, i.e. $\mathbf{m}$ contains all necessary information w.r.t the hidden representation $\mathbf{z}$, the following information bottleneck inequality holds:
\begin{equation}
\label{EQ:IBLBO_APP}
\begin{aligned}
\mathbf{MI}(\mathbf{z}, \mathbf{y})-\lambda \cdot \mathbf{MI}(\mathbf{z}, \mathbf{m}) \geq \operatorname{IBLBO} & \triangleq \\
\mathbb { E } _ {p (\mathbf{m}, \mathbf{y}) } \Bigl[ \mathbb { E } _ {p(\mathbf{z}|\mathbf{m})} \left[\log q (\mathbf{y}|\mathbf{z})\right] - \lambda \cdot \mathbf { KL }  ( p (\mathbf{z}&|\mathbf{m}) \| p (\mathbf{z}) ) \Bigl]
\end{aligned}
\end{equation}
\end{theorem}

\noindent where $p (\mathbf{m}, \mathbf{y})$ is the data distribution, $p(\mathbf{z}|\mathbf{m})$ is the encoder distribution,  $q (\mathbf{y}|\mathbf{z})$ is the variational decoder distribution and $p (\mathbf{z})$ is the specified prior for $\mathbf{z}$. In order to prove the theorem, we first prove two lemmas below.

\newtheorem{lemma}[theorem]{Lemma}
\begin{lemma}{}
Under the assumptions of Theorem \ref{THEO_1}, the following holds:
\label{LEMMA_FIRST}
\begin{equation}
\mathbf{MI}(\mathbf{z}, \mathbf{y}) \geq \mathbb { E } _ {p (\mathbf{m}, \mathbf{y}) } \left[ \mathbb{E} _ {p(\mathbf{z}|\mathbf{m})} \left[ \log q(\mathbf{y}|\mathbf{z})\right] \right]
\end{equation}
\end{lemma}

\begin{proof}

According to the definition of mutual information:

\begin{equation}
\begin{aligned}
\mathbf{MI}(\mathbf{z}, \mathbf{y})= \mathbb { E } _ { p(\mathbf{y}, \mathbf{z}) } \left[ \log \frac{p(\mathbf{y}, \mathbf{z})}{p(\mathbf{y}) p(\mathbf{z})} \right]=\mathbb { E } _  {p(\mathbf{y}, \mathbf{z})} \left[ \log \frac{p(\mathbf{y} | \mathbf{z})}{p(\mathbf{y})} \right]
\end{aligned}
\end{equation}

Observing that the KL-divergence is non-negative, the following inequality holds.

\begin{equation}
\begin{aligned}
\mathbf{KL}[p(\mathbf{y} | \mathbf{z}) || &  q(\mathbf{y} | \mathbf{z})] \geq 0 \implies \\
\mathbb{E} _  { p(\mathbf{y} | \mathbf{z}) } & \left[ \log p(\mathbf{y} | \mathbf{z}) \right] \geq  \mathbb{E} _ { p(\mathbf{y} | \mathbf{z}) } \left [\log q(\mathbf{y} | \mathbf{z}) \right]
\end{aligned}
\end{equation}

Then, the $\mathbf{MI}(\mathbf{z}, \mathbf{y})$ is lower-bounded by:

\begin{equation}
\begin{aligned} 
\mathbf{MI}(\mathbf{z}, \mathbf{y}) & \geq \mathbb{E} _ { p(\mathbf{y}, \mathbf{z}) } \left[ \log \frac{p(\mathbf{y} | \mathbf{z})}{p(\mathbf{y})} \right] \\ &= \mathbb{E} _ {p(\mathbf{y}, \mathbf{z})} \log q(\mathbf{y} | \mathbf{z})-\mathbb{E}  _ {p(\mathbf{y})} \log p(\mathbf{y}) \\ &= \mathbb {E} _ { p(\mathbf{y}, \mathbf{z})} \log q(\mathbf{y} | \mathbf{z})+\mathbf{H}(\mathbf{y}) 
\end{aligned}
\end{equation}

Since $\mathbf{H}(\mathbf{y})$ is the entropy of $\mathbf{y}$, which is a positive constant independent of the hidden encoding $\mathbf{z}$, it can be safely ignored. According to the independence assumption of Theorem \ref{THEO_1}, $p(\mathbf{y}, \mathbf{z})=\int p(\mathbf{m}, \mathbf{y}, \mathbf{z}) \mathrm{d}\mathbf{m} =\int p(\mathbf{m}) p(\mathbf{y} | \mathbf{m}) p(\mathbf{z} | \mathbf{m})  \mathrm{d} \mathbf{m}$, then, a new lower bound of $\mathbf{MI}(\mathbf{z}, \mathbf{y})$ can be deduced:

\begin{equation}
\begin{aligned}
\mathbf{MI}(\mathbf{z}, \mathbf{y}) \geq& \int p(\mathbf{m}) p(\mathbf{y}|\mathbf{m}) p(\mathbf{z}|\mathbf{m}) \log q(\mathbf{y}|\mathbf{z}) \mathrm{d}\mathbf{m}\mathrm{d}\mathbf{y}\mathrm{d}\mathbf{z} \\
=&\mathbb{E} _ {p({\mathbf{m}, \mathbf{y}})} \left[ \mathbb{E} _ { p(\mathbf{z} | \mathbf{m}) } \left [\log q(\mathbf{y} | \mathbf{z}) \right] \right]
\end{aligned}
\end{equation}

\noindent which finishes our proof of Lemma \ref{LEMMA_FIRST}.

\end{proof}

\begin{lemma} 
\label{LEMMA_SECOND}
Under the assumptions of Theorem \ref{THEO_1}, the following holds:
\begin{equation}
\mathbf{MI}(\mathbf{z}, \mathbf{m}) \leq \mathbb { E } _ {p (\mathbf{m}, \mathbf{y}) } \left [\mathbf{KL}\left[ p(\mathbf{z}|\mathbf{m})| p(\mathbf{z})\right] \right]
\end{equation}
\end{lemma}

\begin{proof}

For clarity, with a little misuse of notation, we use $\hat{p}(\mathbf{z})$ to denote the marginal distribution and $p(z)$ as the assumed prior. Observing the fact that 

\begin{equation}
\begin{aligned}
\mathbf{KL}[\hat{p}(\mathbf{z}) || q(\mathbf{z})] \geq 0 \implies \mathbb{E} _ { \hat{p}(\mathbf{z}) } \left[ \log \hat{p}(\mathbf{z}) \right] \geq \mathbb{E} _ { \hat{p}(\mathbf{z}) } \left[ \log q(\mathbf{z}) \right]
\end{aligned}
\end{equation}

The following upper bound can be obtained:

\begin{equation}
\begin{aligned}
\mathbf{MI}(\mathbf{z}, \mathbf{m}) &\leq \int p(\mathbf{m}) p(\mathbf{z} | \mathbf{m}) \log \frac{p(\mathbf{z} | \mathbf{m})}{p(\mathbf{z})} \mathbf{d} \mathbf{m} \mathbf{d} \mathbf{z} \\
&= \int p(\mathbf{m}) p(\mathbf{y}|\mathbf{m}) p(\mathbf{z} | \mathbf{m}) \log \frac{p(\mathbf{z} | \mathbf{m})}{p(\mathbf{z})} \mathbf{d} \mathbf{y} \mathbf{d} \mathbf{m} \mathbf{d} \mathbf{z} \\
&= \mathbb { E } _ {p (\mathbf{m}, \mathbf{y}) } \left [\mathbf{KL}\left[ p(\mathbf{z}|\mathbf{m})| p(\mathbf{z})\right] \right]
\end{aligned}
\end{equation}

\noindent which finishes our proof of Lemma \ref{LEMMA_SECOND}.

\end{proof}

Finally, combining the two bounds of Lemma \ref{LEMMA_FIRST} and \ref{LEMMA_SECOND} finishes the proof for Theorem \ref{THEO_1}.

\section{Multimodal Product-of-experts Encoder with Conditional Independence Assumption}
\label{APP:POE}

By making the conditional independence assumption, i.e., assuming  ($\mathbf{m} _ {i} \perp \mathbf{m} _ {j} | \mathbf{z})  \quad \forall i,j < K$ and $i \neq j$ and using the Bayesian rule, the following equation holds:

\begin{equation}
\label{eq:poe1}
\begin{aligned} p ( \mathbf{z} | \mathbf{m})  &= \frac { p \left( \mathbf{m} | \mathbf{z} \right) p ( \mathbf{z} ) } { p \left( \mathbf{m} \right) }  = \frac { p ( \mathbf{z} ) } { p \left( \mathbf{m} \right) } \prod _ { i = 1 } ^ { K } p \left( \mathbf{m} _ { i } | \mathbf{z} \right) \\ 
& = \frac { p ( \mathbf{z} ) } { p \left( \mathbf{m} \right) } \prod _ { i = 1 } ^ { K } \frac { p ( \mathbf{z} | \mathbf{m} _ { i } ) p \left( \mathbf{m} _ { i } \right) } { p ( \mathbf{z} ) } \propto \frac { \prod _ { i = 1 } ^ { K } p ( \mathbf{z} | \mathbf{m} _ { i } ) } { \prod _ { i = 1 } ^ { K - 1 } p ( \mathbf{z} ) } \end{aligned}
\end{equation}

Moreover, to avoid quotient of probability distributions, we further assume that $p ( \mathbf{z} | \mathbf{m} _ { i })$ can be approximated with $ \tilde { p } ( \mathbf{z} | \mathbf{m} _ { i } ) p ( \mathbf{z} ) ,$ where $\tilde { p } ( \mathbf{z} | \mathbf{m} _ { i })$ is the encoder network  for the $i$th modality, and $p(\mathbf{z})$ is the prior. Then, the following simplification of Eq. \ref{eq:poe1} can be made:

\begin{equation}
\label{eq:poe2}
\begin{aligned} p ( \mathbf{z} | \mathbf{m}) & \propto \frac { \prod _ { i = 1 } ^ { K } p ( \mathbf{z} | \mathbf{m} _ { i } ) } { \prod _ { i = 1 } ^ { K - 1 } p ( \mathbf{z} ) } \approx &\\
&\frac { \prod _ { i = 1 } ^ { K } [ \tilde { p } ( \mathbf{z} | \mathbf{m} _ { i } ) p ( \mathbf{z} ) ] } { \prod _ { i = 1 } ^ { K - 1 } p ( \mathbf{z} ) } = p ( \mathbf{z} ) \prod _ { i = 1 } ^ { K } \tilde { p } ( \mathbf{z} | \mathbf{m} _ { i } ) \end{aligned}
\end{equation}

\noindent which finishes the deduction for the product-of-experts based multimodal encoder utilized in our paper.

\end{document}